\documentclass[sigconf,authorversion,nonacm]{acmart}

\usepackage{booktabs} % For formal tables

\setcopyright{acmcopyright}
\acmDOI{xx.xxx/xxx_x}

\acmISBN{979-8-4007-0243-3/24/04}

\acmConference[SAC'24]{ACM SAC Conference}{April 8 –April 12, 2024}{Avila, Spain}
\acmYear{2024}
\copyrightyear{2024}

\acmArticle{4}
\acmPrice{15.00}

\usepackage{xspace} % Needed for spacing of newcommands
\usepackage{amsfonts} % Needed for math fonts
\usepackage{amsmath}
\usepackage{amsthm}

\usepackage{decision-table} % Needed for DMN tables

\newcommand{\IDP}{\textsc{IDP}\xspace}
\newcommand{\ZT}{\textsc{Z3}\xspace}
\newcommand{\DMN}{\textsc{DMN}\xspace}
\newcommand{\EDMN}{\textsc{eDMN}\xspace}
\newcommand{\OEL}{\textsc{OEL}\xspace}
\newcommand{\FODOT}{\textsc{FO(.)}\xspace}
\newcommand{\IDPZT}{\textsc{\IDP-\ZT}\xspace}

\newcommand{\EBD}{\textit{ebd}\xspace}

\newtheorem{definition}{Definition}
\newtheorem{theorem}{Theorem}
\newtheorem{corollary}{Corollary}[theorem]

\newcommand{\ignore}[1]{}

\newcommand{\struct}{\mathfrak{A}}
\newcommand{\voc}{\Sigma}
\newcommand{\PCvoc}{\hat{\voc}}
\newcommand{\strclass}[1]{\mathcal{C}(#1)}
\newcommand{\pow}[1]{\mathcal{P}(#1)}
\newcommand{\dmntr}[1]{\mathcal{T}(#1)}

\DeclareMathOperator*{\Fagg}{fagg}
\DeclareMathOperator*{\Fopt}{fopt}

\usepackage{bibentry}
\begin{document}

\title{An epistemic logic for modeling decisions in the context of incomplete knowledge}
%Epistemic aspects of decision making in the context of incomplete knowledge \\ The decision logic in the context of incomplete knowledge \\ An epistemic logic for modeling decisions in the context of incomplete knowledge}

\titlenote{This research received funding from the Flemish Government under the ``Onderzoeksprogramma Artifici\"ele Intelligentie
(AI) Vlaanderen'' programme.}
%\titlenote{Produces the permission block, and copyright information}
%\subtitle{Extended Abstract}
%\subtitlenote{The full version of the author's guide is available as
%  \texttt{acmart.pdf} document}
  
\renewcommand{\shorttitle}{An epistemic logic for modeling decisions in the context of incomplete knowledge}

\author{\DJ or\dj e Markovi\'c}
\orcid{1234-5678-9012}
\affiliation{%
  \institution{KU Leuven, Belgium}
  \streetaddress{Celestijnenlaan 200A}
  \city{} 
  \country{}
}
\email{dorde.markovic@kuleuven.be}

\author{Simon Vandevelde}
\orcid{1234-5678-9012}
\affiliation{%
  \institution{KU Leuven, Belgium}
  \streetaddress{Celestijnenlaan 200A}
  \city{} 
  \country{}
}
\email{s.vandevelde@kuleuven.be}

\author{Linde Vanbesien}
\orcid{1234-5678-9012}
\affiliation{%
  \institution{KU Leuven, Belgium}
  \streetaddress{Celestijnenlaan 200A}
  \city{} 
  \country{}
}
\email{linde.vanbesien@kuleuven.be}

\author{Joost Vennekens}
\orcid{1234-5678-9012}
\affiliation{%
  \institution{KU Leuven, Belgium}
  \streetaddress{Celestijnenlaan 200A}
  \city{} 
  \country{}
}
\email{joost.vennekens@kuleuven.be}

\author{Marc Denecker}
\orcid{1234-5678-9012}
\affiliation{%
  \institution{KU Leuven, Belgium}
  \streetaddress{Celestijnenlaan 200A}
  \city{} 
  \country{}
}
\email{marc.denecker@kuleuven.be}

\renewcommand{\shortauthors}{D. Markovic et al.}

\begin{abstract}
Substantial efforts have been made in developing various Decision Modeling formalisms, both from industry and academia.
A challenging problem is that of expressing decision knowledge in the context of incomplete knowledge.
In such contexts, decisions depend on what is known or not known.
We argue that none of the existing formalisms for modeling decisions are capable of correctly capturing the epistemic nature of such decisions, inevitably causing issues in situations of uncertainty.
This paper presents a new language for modeling decisions with incomplete knowledge.
It combines three principles: stratification, autoepistemic logic, and definitions.
A knowledge base in this language is a hierarchy of epistemic theories, where each component theory may epistemically reason on the knowledge in lower theories, and decisions are made using definitions with epistemic conditions.
\end{abstract}


\maketitle

% -----------------SECTION - Introduction
\section{Introduction}
Decision-making is an everyday human activity, from trivial ones (e.g., crossing the street) to complex domain-specific decisions (e.g., sell the stock or not?).
As such, decision-making and modeling play an important role in artificial intelligence systems.
They also have a significant place as a philosophical topic~\cite{sep-decision-theory}, while theoretical aspects are well-studied as a branch of probability theory, named Decision Theory~\cite{peterson2017introduction}.
Modeling of decision processes is studied in different fields.
In Machine Learning, decision trees~\cite{kotsiantis2013decision} are used for representing decisions extracted from data~\cite{quinlan1990decision}.
In Knowledge Representation (KR), the Decision Model and Notation~(\DMN)~\cite{DMN:2021} is a broadly accepted standard with many implementations\footnote{An overview of the most relevant tools is available in~\cite{vandevelde2021leveraging}.}.

Decision Theory studies essential properties of decision-making, in terms of their \emph{utilities} -- a benefit of decisions in different possible worlds proportional to their probabilities.
Focusing on a generalized, theoretical notion of decisions, Decision Theory is not concerned with Knowledge Representation aspects of formally representing decisions.
This is reflected by the total absence of epistemic constructs from the models.
According to philosophical studies~\cite{sep-decision-theory}, decisions are made by an agent in accordance with its epistemic state.
Nevertheless, in Decision Theory, the epistemic nature of decisions is abstracted and brought to the implicit level as part of the utility function.
Decision modeling approaches, i.e., decision trees or \DMN, define decisions as consequences of the objective state of affairs, totally neglecting their epistemic character.
This incompatibility results in a \emph{semantic mismatch}, where the knowledge and its representation are not aligned.
This gives rise to many practical problems, mainly reflected by a lack of \emph{semantic clarity} \cite{clark1996requirements} and \emph{elaboration tolerance} \cite{mccarthy1998elaboration} as we shall demonstrate.

Problems caused by \emph{semantic mismatch} are easiest to observe in situations when a particular decision is to be made in case of ignorance.
This problem is referred to as a \emph{decision under uncertainty\footnote{The term \emph{decision under uncertainty} is overloaded -- it can also stand for a lack of information on probabilities of possible worlds -- and hence we stress that in this paper we are concerned with uncertainty as a form of ignorance about the exact state of affairs. For an analysis of different reasons causing uncertainty, see~\cite[Chapter~12]{russell2022artifcial}.}}.
Consider a decision problem where one of the parameters is the marital status of a person.
Usually, marital status takes one of the values married, single, divorced, or widowed.
Suppose that the marital status is not known, a decision can still be made in some cases.
For example, in most countries, a person can get married if they are not already married, and surpasses a minimum age. 
For this decision to be made, any partial knowledge about the marital status that entails that a person is not married is sufficient (e.g., it is known that the person is divorced or widowed).

The approach in Decision Theory relies on a clever design of utility function by providing scores to decisions for each possible world.
Then, making decisions under ignorance is done by applying particular \emph{optimization criteria}~\cite[Chapter~3]{peterson2017introduction} on the set of worlds considered possible according to the agent's knowledge.
This approach, while capable of capturing incomplete knowledge, lacks semantic clarity due to the unclear meaning of utility functions (i.e., utility function has no meaning without optimization criterion). 
On the positive side, optimization criteria can be assigned epistemic interpretation.
Another issue with this approach is that creating a specification is cumbersome (requiring score for each decision per possible world) and hence difficult maintenance, an important aspect of elaboration tolerance.
The complexity of updating utility functions originates in its use with optimization criteria, i.e., the new function should satisfy the same properties as the old one which can require radical changes for the tiniest modifications.
On the side of modeling formalisms, in particular decision trees and \DMN, the problem is much bigger since it is impossible to express that something is not known.
In attempts to work around this, knowledge engineers are forced to introduce artificial objects in the domain of values representing a special case when the value is unknown. 
This suddenly changes the meaning of the symbol being assigned a value from an objective to an epistemic meaning.
This approach is problematic from the semantic clarity perspective as the informal meaning of the statements is overloaded.
Furthermore, this method falls short of modeling different levels of ignorance, as it can only represent exact knowledge or complete ignorance, which is often insufficient.
Handling this problem would require further domain corruption, ultimately leading to an extremely elaboration-intolerant approach.

The problems of \emph{semantics mismatch} are resolved with epistemic logic.
However, decision models are more specific than the general epistemic theories, namely, there is a special link between the environmental (parameter) variables and decision variables\footnote{More about the general notion of this property is available in \cite{carbonnelle_vennekens_denecker_bogaerts_2023}.}. 
Furthermore, modeling decision problems often requires the specification of an exact epistemic state of an agent -- the problem of ``only knowing\footnote{More about only knowing is available in~\cite{levesque1990all}.}''.
Taking these properties into account, we show that Ordered Epistemic Logic (\OEL)~\cite{vlaeminck2012ordered} makes a perfect fit for the purposes of modeling decisions. 
Accordingly, we will show that \OEL is capable of modeling optimization approaches from Decision Theory and \DMN decision models; reflected in contributions:
(1) Proof that \OEL can correctly represent state-of-the-art (\DMN and Decision Theory) decision models, (2) The introduction of a new language \EDMN -- epistemic \DMN, (3) Translational semantics of \EDMN, using \OEL, (4) Proof that \EDMN can correctly represent state-of-the-art (\DMN and Decision Theory) decision models, (5) Defining syntactical fragment of \OEL for decision modeling, (6) Developing an \OEL KBS and translation of \EDMN to \OEL.

%\begin{enumerate}
%    \item Proof that \OEL can correctly represent state-of-the-art (\DMN and Decision Theory) decision models.
%    \item New language \EDMN -- epistemic \DMN.
%    \item Translational semantics of \EDMN, using \OEL.
%    \item Proof that \EDMN can correctly represent state-of-the-art (\DMN and Decision Theory) decision models.
%    \item Defining syntactical fragment of \OEL for decision modeling.
%    \item Developing an \OEL KBS and translation of \EDMN to \OEL.
%\end{enumerate}

The rest of the paper is structured in the following sections: (2) Preliminaries, (3) a formalization of epistemic decisions and relation to approaches from Decision Theory and \DMN, (4) \OEL as a modeling language of epistemic decisions, (5) epistemic \DMN, (6) decision modeling fragment of \OEL, (7) \OEL and \EDMN implementation, (8) examples of \EDMN, (9) related work, and (10) conclusion.

% -----------------SECTION - Preliminaries
\section{Preliminaries}

This section provides an overview of different existing concepts used in this paper. 
Our approach is focused on logic-based modeling of knowledge~\cite{jackson1989logic}, hence first we introduce those, and further define other formal objects in terms of the same concepts.

\subsection{Logic}
Following are the formal definitions of vocabulary and structure\footnote{For more details, see \cite{sep-modeltheory-fo}.}. 

\begin{definition}
    A vocabulary $\voc$ is a set consisting of sort symbols $s$, predicate symbols $p$, and function symbols $f$. Additionally, a vocabulary contains a mapping from predicate and function symbols to tuples of sort symbols. Predicate symbols can be mapped to an empty tuple (propositional symbol), while function symbols must be mapped to a tuple of at least one element (constant symbols).
    \label{def:ep-dec-voc}
\end{definition}

Following is the definition of a structure, object assigning mathematical objects to symbols from a vocabulary.
\begin{definition}
    A structure $\struct$ over vocabulary $\voc$ assigns:
    \begin{itemize}
        \item A nonempty set $s^\struct$ to each sort symbol $s$ in $\voc$. 
        \item Per predicate symbol $p$ of sort $(s_1, \dots, s_n)$: a function $p^\struct$ mapping elements of $s_1^\struct \times \dots \times s_n^\struct$ to $true$ or $false$.
        \item Per function symbol $f$ of sort $(s_1, \dots, s_n, s)$: A function $f^\struct$ mapping elements of $s_1^\struct \times \dots \times s_n^\struct$ to elements of $s^\struct$.
    \end{itemize}
\end{definition}

A term, atom, and formulae, of many-sorted logic, are defined inductively in a usual way, as well as their value in a structure \cite{sep-logic-many-sorted}.
$\strclass{\voc}$ denotes the class of all possible structures over $\voc$. 
A structure is \emph{partial} if it assigns partial value to some symbols from the vocabulary.
Following is the definition of a vocabulary of propositional and constant symbols and fixed interpretation of sorts.

\begin{definition}
    A propositional vocabulary with constant symbols and interpreted types  $\PCvoc$ (P/C vocabulary for short), is a vocabulary that contains only propositional and constant function symbols and all sort symbols have the same (finite) interpretation in all structures. 
    \label{def:pc-voc}
\end{definition}

\subsection{Inductive definitions}
First-order logic extended with inductive definitions FO(ID) was shown to be a valuable Knowledge Representation language \cite{Bruynooghe2016,Denecker2000}.
\begin{definition}
    Given that $p(\bar{t})$ is an FO atom (applied to a tuple of terms $\bar{t}$) and $\phi$ a FO formula over the same vocabulary $\voc$, an FO(ID) definition is a set of rules of the form $\forall \bar{x} : p(\bar{t}) \leftarrow \phi$.
\end{definition}
Semantics of definitions is not trivial, and it is defined using the well-founded semantics \cite{denecker2014well}.
However, all the definitions occurring in this work are monotone and not inductive, and hence could be translated to a set of equivalences, Clark's first order completion semantics \cite{DBLP:conf/adbt/Clark77} (with slight adjustments).
Nevertheless, definitions are syntactic constructs clearly separating definition parameters from defined concepts\footnote{This property is used in relevance inference~\cite{jansen2016relevance}.}, property suitable for separation of environmental and decision variables.

\subsection{Decision Model and Notation}

\begin{figure}
\small
\dmntable{Salutation}{U}{Gender, Marital status}{Salutation}
             {Male, -, Mr, 
             Female, Single, Ms, 
             Female, Married, Mrs}
\vspace{-2em}
\caption{\DMN decision table - Greeting example}
\label{fig:salutation}
\end{figure}

\DMN notation specifies formal representation of decisions as a special kind of table, named \emph{decision tables}; example in Figure~\ref{fig:salutation}. 
In such a table, the value of the output variable(s) (in blue, right) are defined by the values of the input variables (in green, left).
We also refer to these variables as decision and environment respectively.
Rows contain values for input variables and decisions in the last column. 
A symbol ``-'' in a value cells denotes that the value does not matter.w
The behavior of a table is defined by its \textit{hit policy}, as denoted in the top-left corner.
In the example, ``U'' stands for the unique hit policy, meaning that for any input there should be exactly one row matching it. 
The academic studies of \DMN formal semantics are found in \cite{calvanese2016semantics,calvanese2019semantic,cDMN,markovic2022semantics}.
Tables will be formally represented as:
\begin{definition}
    A \DMN decision table $T$ consists of environment variables $e_1,\dots,e_n$, a decision variable $d$, constraints\footnote{Often referred to as expressions in the S-FEEL language.} $C_{11},\dots,C_{nm}$, decision assignments $A_1,\dots,A_m$, and a hit-policy $hp$ (index $1,\dots,n$ represents columns and $1,\dots,m$ rows). Additionally, $C(e)$ stands for the first-order expression of the application of constraint $C$ on the variable $e$ and $A(d)$ for the assignment $A$ to the decision variable $d$.
    \label{def:dmn}
\end{definition}

The example presented in the table from Figure~\ref{fig:salutation} originates from a \DMN challenge \cite{DMN-challenge:2016}, presenting the problem of decision-making with ignorance. 
Namely, a decision can be made if the gender is known to be ``Male'' even though marital status is unknown.
The \DMN standard natively does not support this kind of problem while some tools\footnote{\url{https://camunda.com/}, \url{https://signavio.com/}} like Camunda and Signavio are capable of performing a simplified form of decision-making by providing all possible decisions in the case of unknown variables.
The greeting problem will serve as a running example throughout this paper.

\subsection{Utility}
The notion of utility is coming from the Decision Theory, where it is used to represent how beneficial certain decisions/actions are in a particular state of affairs.
Following is the definition of the utility function and its application on the running example is in Table \ref{table:utility-table}. 

\begin{definition}[Decision utility function]
    Given an environment and decision P/C vocabularies $\PCvoc_e$ and $\PCvoc_d$ a decision utility function $f_u$ maps pairs of structures $(\struct_e, \struct_d)$ (from $\strclass{\PCvoc_e} \times \strclass{\PCvoc_d}$) to an ordinal number. 
\end{definition}
%Consider the running example extended with a new decision ``Customer'', which is less precise than any other salutation but on the other hand is suitable for all of them.  
%The running example can be captured by the utility function presented in Table \ref{table:utility-table}.
\begin{table}
\caption{Utility function of the greeting example.}
\label{table:utility-table}
\begin{tabular}{c|c|c|c|c}
    \toprule
    & Male & Male & Female & Female \\
    & Single & Married & Single & Married \\
    \hline
    Mr & 1 & 1 & 0 & 0 \\
    Mrs & 0 & 0 & 0 & 1 \\
    Ms & 0 & 0 & 1 & 0 \\
    %$\{Custorme\}$ & 1 & 1 & 1 & 1 \\
    \bottomrule
\end{tabular}
\end{table}

\subsection{Ordered Epistemic Logic}
Ordered Epistemic Logic (\OEL) was introduced independently in \cite{DBLP:conf/aaai/Konolige88,DBLP:conf/birthday/DeneckerVVWB11} as a language capable of expressing many interesting KR epistemic example while providing lower complexity compared to Autoepistemic Logic~\cite{moore1985possible}.
%special kind of epistemic logic equipped with \update{epistemic closure} and knowledge stratification.
Before formally defining \OEL we shall briefly introduce a notion of epistemic state\footnote{For more details, the reader is invited to check \cite{sep-logic-epistemic}.}.
\begin{definition}
    Given a first-order vocabulary $\voc$, any collection of structures $E$ such that $E \subseteq \strclass{\voc}$ is an epistemic state.  
    \label{def:epistemic-state}
\end{definition}
Intuitively, singleton epistemic states represent a state of absolute knowledge (i.e., everything is known), while a set of all possible worlds corresponds to the state where nothing is known, the empty set represents inconsistency in knowledge.
Opposed to standard epistemic logic \cite{sep-logic-epistemic} the $K$ operator in \OEL is not self-referential and can refer only to the knowledge lover in the hierarchy.
\begin{definition}
    A set of FO(ID) theories $\mathcal{T}$ (over the same vocabulary $\voc$) is an order epistemic theory iff:
    \begin{enumerate}
        \item Each theory in $\mathcal{T}$ is composed of FO(ID) logic formulas defined in a standard way with one additional rule: if $\psi$ is a formula and $T \in \mathcal{T}$ then $K[T][\psi]$ is a formula; with one exception, $K$ operators are not appearing in the head of definition rules.  
        \item There exists a strict partial order $<$ such that if there is $K[T'][\psi]$ in theory $T$ then $T' < T$.
    \end{enumerate}
    \label{def:oel}
\end{definition}

The stratification of theories in \OEL provides semantic reduction\footnote{For more details, reader is referred to the paper \cite{vlaeminck2012ordered}.} of the $K$ operator to the standard first-order logic.

\begin{definition}
    Let $\mathcal{T}$ be an \OEL theory over vocabulary $\voc$, and $\struct$ a (partial) structure over $\voc$. Structure $\struct$ satisfies theory $T$ from $\mathcal{T}$, in symbols $\struct \models T$, iff:
    \begin{itemize}
        %\item For every $T' < T$, it holds that $\struct \models T'$. \djordje{Optional?}
        \item For atom $p(\bar{t})$ from $T$, $\struct \models p(\bar{t})$ iff $p^\struct(\bar{t}^\struct) = true$.
        \item The inductive cases for $\land, \lor, \neg, \exists, \forall$ are defined as usual.
        \item For modal operator $K[T'][\psi]$, $\struct \models K[T'][\psi]$, iff $\struct' \models \psi$ for all $\voc$-structures $\struct'$ extending $\struct$ such that $\struct' \models T'$.
    \end{itemize}
    \label{def:oel-sem}
\end{definition}
 
Note that partial order ensures that there is a theory at the ``bottom'' that is free of epistemic operators, which ensures the reduction to standard FO semantics.
Following is the \OEL representation of the running example, where theory $T$ represents the user input (what is known about gender and marital status) while the decision table (Figure \ref{fig:salutation}) is modeled\footnote{We use obvious abbreviations for variables and values.} as a definition.
\[\left\{
    \begin{array}{c}
    sal() = Mr \leftarrow K[T][gen() = Ma].\\
    sal() = Ms \leftarrow K[T][gen() = Fe \land mar() = Sin].\\
    sal() = Mrs \leftarrow K[T][gen() = Fe \land mar() = Mar].\\
    \end{array}
\right\}\]

% -----------------SECTION - Optimal and epistemic decisions
\section{Generalization of different decision models}

In this section we present formal generalization of \DMN decision tables and optimization methods found in Decision Theory\footnote{In this work we focus on exact decision functions, meaning that there exists at most one decision for any possible state of affairs. For practical systems, it is common to impose such criteria, meaning that the agent is never indifferent between the options.} in terms of logical objects introduced earlier. 
Furthermore, we introduce a general notion of epistemic decision function capable of capturing both of the other formalisms.

%Making decisions is an activity that often has to be performed with incomplete knowledge about the state of affairs.
%This point was already made in the airport accident and greeting example.
%It is obvious that the \DMN function as formalized in Definition \ref{def:dmn-dec-function} is not capable of capturing these cases.
%One of the standard approaches is to employ some optimal reasoning and find the most appropriate decision.
%In this section, we generalize this idea as an optimal decision\footnote{In this work we focus on exact decision functions, meaning that there exists at most one decision for any possible state of affairs. For practical systems, it is common to impose such criteria meaning that the agent is never indifferent between the options.} and later introduce a more general idea of epistemic decision.
%Finally, we shall show that optimal decisions are sub-class of epistemic decisions.

\subsection{DMN decisions}

A \DMN table can be represented, abstractly, as a function defined in terms of logical objects, i.e., structures and vocabularies.
An important observation is that environment/decision variables of a \DMN table are either propositional or constant function symbols, and furthermore ranging over a finite set of values.
%This property will be important for proving some of the theorems later.

\begin{definition}[\DMN Decision function]
    Given an environment and decision P/C vocabularies $\PCvoc_e$ and $\PCvoc_d$ a \DMN table is a function $f_d$ mapping environment structures $\struct_e$ from $\strclass{\PCvoc_e}$ to decision structures $\struct_d$ from $\strclass{\PCvoc_d}$.
\label{def:dmn-dec-function}
\end{definition}

The environment vocabulary of the running example consist\footnote{Names of sorts start with capital letters, constant symbols are denoted with $()$ after the name. Abbreviations are used for constant names (e.g., ``gen'' stands for ``gender'').} of two sorts $Gender$ and $MStatus$, and two constants $gen()$ of sort $Gender$ and $mar()$ of sort $MStatus$.
Decision vocabulary consist of sort $Salutation$ and a constant $sal()$ of sort $Salutation$.  
An example of environment structure $\struct_e$ could be:
\[\left\{
    \begin{array}{c}
    Gender = \{Male,Female\}; MStatus = \{Single,Married\};\\
    gen() = Male; mar() = Single
    \end{array}
\right\}\]

When clear from the context, a shorter notation, omitting the sorts and constant names, will be used: $\{Male; Single\}$ (even shorter for the running example  $\{Ma; Sin\}$).
The decision function of the running example is hence: 
\[\left\{
    \begin{array}{c}
    \{Ma; Sin\} \rightarrow \{Mr\}, \{Ma; Mar\} \rightarrow \{Mr\}, \\ 
    \{Fe; Sin\} \rightarrow \{Ms\}, \{Fe; Mar\} \rightarrow \{Mrs\}
    \end{array}
\right\}\]

\subsection{Optimal decisions}
As already introduced, a common approach in Decision Theory for reasoning under uncertainty is to use some optimality criterion to derive a suitable decision provided an adequate utility function.
\cite{peterson2017introduction} provides an overview of such functions: maximin, leximin, optimist-pessimist, minimax regret, etc.
In the following definition, we generalize this approach, allowing us to study many other such optimality criterion.

\begin{definition}[Optimal decision function]
    Given environment and decision P/C vocabularies $\PCvoc_e$ and $\PCvoc_d$, a set of decisions $D$, and a decision utility function $f_u$, an optimal decision function $f_o$ for an epistemic state $E$ is defined by a pair of functions $\Fopt$ and $\Fagg$ in the following way:
    \[f_o(E) = \Fopt_{d \in D}(\Fagg_{w \in E} f_{u}(w,d))\]
    Where $\Fagg$ is an aggregate function mapping a set of utility values (for the decision $d$ in worlds $w$) to a utility value, and $\Fopt$ is selecting a decision with an optimal utility value (obtained from aggregation).
\end{definition}

For example, the \textit{maximin} principle selects the best of all worst cases and is defined as $\max_{d \in D}(\min_{w \in E} f_{u}(w,d))$.
Applied to the running example on the worlds $\{\{Ma, Sin\}, \{Ma, Mar\}\}$ (i.e., gender known to be male and marital status is unknown), \textit{maximin} will derive ``Mr'' as the best choice based on the utility function from the Table \ref{table:utility-table}.
This is the case because decisions $Mrs$, and $Ms$ have score $0$ in at least one of the possible worlds, and hence this would be the result after the minimization of the value.
On the other hand, decision $Mr$ has a score of $1$ in both worlds and hence, its score will be $1$, which is the maximum among other scores.

\subsection{Epistemic decisions}
As already motivated, decisions are made according to an epistemic state of an agent.
Formalization of such an idea can be realized by an \emph{epistemic decision function}, mapping epistemic states to decisions.

\begin{definition}[Epistemic decision function]
    Given an environment and decision P/C vocabularies $\PCvoc_e$ and $\PCvoc_d$ an epistemic decision function $f_e$ is a (partial) map from collections of structures $E \subseteq \strclass{\PCvoc_e}$ to the decision structures $\struct_d \in \strclass{\PCvoc_d}$.
    \label{def:ep-dec-function}
\end{definition}

In the running example, an epistemic state where gender is known to be ``male'' and marital status unknown, formally represented as $\{\{Ma, Sin\}, \{Ma, Mar\}\}$, would be mapped to the decision $\{Mr\}$.
Since an epistemic function can be any function, it is clear that it subsumes the notion of \emph{optimal decision function} and \emph{\DMN decision function}; formally stated in the following theorem.

\begin{theorem}
    Each \DMN decision function $f_d$ and optimal decision function $f_o$ are an epistemic decision function.
\label{the:oed-as-edf}
\end{theorem}
\begin{proof}
    Given an optimal decision function $f_o$ an epistemic decision function $f_e$ is defined as: For each epistemic state $E$ from $\strclass{\PCvoc_e}$: $f_e(E) = f_o(E)$. 
    Given a \DMN decision function $f_d$ an epistemic decision function $f_e$ is defined as: For each singleton epistemic state $E$: $f_e(E) = f_d(E)$ and undefined for all others. 
\end{proof}

Theorem \ref{the:oed-as-edf} is a trivial one, and it shows that any DMN decision table or optimal decision function can be represented as an epistemic decision function, making it a good generalization model.
One of the reasons for successful application of optimization approaches on decision-making is that decision models are following some rationality principle.
In Decision Theory this is reflected in the preference axiom\footnote{For discussion see \cite[Chapter~8]{peterson2017introduction}.}, stating that the preference relation is transitive, complete, and continuous. 
While we believe that the relation between the two should be investigated in another direction, i.e., how epistemic decision functions correspond to optimization approaches and how they can be compiled into one, this paper focuses on languages suitable for modeling epistemic decisions, and hence this problem will not be addressed in this work.
 
% -----------------SECTION - Modeling epistemic decisions with OEL
\section{Modeling decisions with \OEL}

The fact that decisions are made in an epistemic state of an agent advocates the use of epistemic logic for modeling such. 
Further, the stratification of variables in environmental and decision ensures the requirements for the use of ordered epistemic logic.
In the previous section, the generalized notion of decision model was formalized ultimately as an epistemic decision function.
Hence, to show the suitability of \OEL for modeling decision problems is to show that any epistemic function can be modeled as an \OEL theory. 

\begin{theorem}
    Given an environment and decision P/C vocabularies $\PCvoc_e$ and $\PCvoc_d$, and an epistemic decision function $f_e$, there exists an \OEL theory $T_e$ (over $\PCvoc_e \cup \PCvoc_d$) such that for each epistemic state $E$:
    \[f_e(E) = \struct_d \Leftrightarrow Mod(T_e,T_E) = \{\struct_d\}\]
    \[f_e(E) = \bot \Leftrightarrow Mod(T_e,T_E) = \{\}\]
    Where $Mod(T_1,\dots,T_n)$ denotes a class of models for a sequence of \OEL theories $T_1,\dots,T_n$ taking theory $T_1$ as the top one. $T_E$ is a first-order theory such that $Mod(T_E) = E$. $\bot$ stand for undefined.
    \label{the:edf-as-oel}
\end{theorem}

\begin{proof}
Given that both $\PCvoc_e$ and $\PCvoc_d$ consist only of finitely many constants ranging over finite domains (according to Definition \ref{def:pc-voc}) it follows that the class of all structures $C(\PCvoc_e \cup \PCvoc_d)$ is a finite set. Hence, the set of all possible epistemic states $\pow{C(\PCvoc_e \cup \PCvoc_d)}$ is finite as well (where $\pow{}$ stands for a power set).

Let $c^E$ denote a set of values that constant $c$ is assigned by structures in $E$, formally: $c^E = \{v \mid \exists \struct \in E : c^\struct = v\}$.
Then, each epistemic state $E$ is expressible as an \OEL statement:
\[\bigwedge_{c \in \PCvoc_e} K[T_E][\bigvee_{v \in c^E} c = v]\]
where $T_E$ is the theory that actually expresses the state of affairs of the environment $E$.

Finally, if the epistemic decision function $f_e$ maps some epistemic state $E$ to a decision structure $\struct_d$, formally $f_e(E) = \struct_d$, and $d$ is a decision constant symbol from $\PCvoc_d$, and $v_d$ is a value assigned to $d$ by $\struct_d$, formally $d^{\struct_d} = v_d$, then the ordered epistemic logic theory shall contain the following definition rule:
\[d = v_d \leftarrow \bigwedge_{c \in \PCvoc_e} K[T_E][\bigvee_{v \in c^E} c = v]\]
Otherwise, i.e., when $f_e(E) = \bot$, no rule is added to the theory.

Let the theory $T_e$ consist of definition composed of such rules for each epistemic state $E$.
Such a theory is finite since set of all possible epistemic states is also a finite set.
By construction such theory satisfies the constraint from the theorem.
This is true because for each epistemic state $E$ (mapped by the function $f_e$) there is exactly one rule whose right-hand side is satisfied and hence the only model is where the left-hand side is satisfied as well, and the left-hand side is constructed by the function $f_e$. 
In case when $E$ is mapped to $\bot$, no rule is applicable and hence theory has no models.
\end{proof}

Combining theorems \ref{the:oed-as-edf} and \ref{the:edf-as-oel} results in the following corollary. 

\begin{corollary}
    Every \DMN decision function and optional decision function can be modeled as an \OEL theory.
    \label{cor:odf-as-oel}
\end{corollary}

With this result, the first contribution mentioned in the list from the Introduction is finished.
The practical merits of these results are suggesting that many decision modeling techniques are expressible as \OEL theory. 
This means that any such language could make use of \OEL solvers as an engine. 

\section{Epistemic \DMN}

Due to the lack of epistemic language constructs, \DMN has no means of expressing more complex decision functions such as epistemic ones (from Definition \ref{def:ep-dec-function}).
The class of decision models expressible by \DMN is characterized in Definition \ref{def:dmn-dec-function}, where only single structures (i.e., exact epistemic states) are mapped to decisions. 
This is defacto following from the \DMN standard as specified by OMG~\cite{DMN:2021}.
Consequently, \DMN is not well suited for modeling decisions under uncertainty, and hence all tools developed based on this standard are lacking this property.
To address this inadequacy we propose an extension of \DMN with the epistemic operator.
We shall call it \EDMN, for epistemic \DMN.

Towards the actual extended notation, we first define a translational semantics of standard \DMN by translating decision tables to \OEL theories. 
Recall, Definition \ref{def:dmn}, $C(e)$ is a first-order expression of \DMN constraint $C$ on the variable $e$, and similar for the decision assignment $A(d)$.

\begin{definition}
    A \DMN table $T_d$ consisting of environment variables $e_1,\dots,e_n$, a decision variable $d$, constraints $C_{11},\dots,C_{nm}$, decision assignments $A_1,\dots,A_m$, and hit-policy\footnote{Due to space constraints, we define translation only for ``Any'' hot-policy since it directly corresponds to definition. Other hit policies require additional axioms.} ``Any'' is translated to \OEL theory, denoted as $\dmntr{T_d}$, containing the following definition:
    \[\begin{array}{c}
        \left\{
        \begin{array}{c}
        A_1(d) \leftarrow K[T][C_{11}(e_1)] \land \dots \land K[T][C_{n1}(e_n)].\\
        \dots\\
        A_m(d) \leftarrow K[T][C_{1m}(e_1)] \land \dots \land K[T][C_{nm}(e_n)].\\
        \end{array}
        \right\}
    \end{array}\]
    Where $T$ is a theory expressing information about environment variables (i.e., input values provided by the user).
    \label{def:dmn-to-oel}
\end{definition}

The translational semantics as proposed here suggest that every constraint in a \DMN decision table is to be interpreted as epistemic.
Aligned with the motivation for this paper, this is the first extension of \DMN which has no impact on syntax but rather on the informal interpretation of decision tables. 
Additionally, the language is extended with the following two constructs:
\begin{itemize}
    \item Two constraints can be connected with logical $\lor$ operator\footnote{Note that \DMN specifies comma ``,'' operator, usually interpreted as or, while actually it stands for connecting two rows, in other words comma operator would be translated to $\lor$ but outside of the $K$ operator!}, $C_1 \lor C_2$ and they stand for $K[T][C_1(e) \lor C_2(e)]$ for some variable $e$.
    \item Constraint $\neg K$ expressing that the value of some variable $e$ is not known. If value of variable $e$ ranges over sort $s_e$ this operator is translated as: $\forall v \in s_e: \neg K[T][e = v]$. Additionally, $\neg K[C]$ stands for $\neg K[T][C(e)]$.
\end{itemize}
Given the new language constructs and their translations to \OEL theory, the semantics of an \EDMN table is defined as follows.

\begin{definition}
    Given a \DMN decision table $T_d$ over environment variables $e_1,\dots,e_n$, and decision variable $d$, and its translation to \OEL theory $\dmntr{T_d}$,  the value of decision derived by the \DMN table is defined as: $d^\struct$ if $Mod(\dmntr{T_d}, T_E) = \{\struct\}$ and $\bot$ if $Mod(\dmntr{T_d}, T_E) = \{\}$, where $T_E$ is a first-order theory specifying (potentially partially) values of environment variables.
    \label{def:dmn-semantics}
\end{definition}

Note that the translation of \DMN table to \OEL has only one model, if satisfiable.
This is because decisions are defined and all parameters of these definitions are epistemic and hence always true or false (this property is formalized in an upcoming section in Theorem \ref{the:ebd-unique}).
However, it is possible for a decision to be overdefined (i.e., multiple rules with different heads are satisfied) or undefined (i.e., no rule is satisfied).
These properties are not desired, and such decision tables are considered wrong.
A detailed investigation is made in \cite{calvanese2016semantics} where overdefindness corresponds to hit policy violation and undefinedness to violation of completeness property.
In both cases, \OEL theory would be unsatisfiable due to the definitional nature of the theory, which might not be the case in some other translations, as we shall see in the next section.

Definitions \ref{def:dmn-to-oel} and \ref{def:dmn-semantics} together are fulfilling the second (new language \EDMN) and third (translational semantics) point of the contribution list.
It remains to formalize and prove the fourth contribution point by showing that \EDMN can express any epistemic decision function (Definition \ref{def:ep-dec-function}). 

\begin{theorem}
    Given an environment and decision P/C vocabularies $\PCvoc_e$ and $\PCvoc_d$, and an epistemic decision function $f_e$, there exists an \EDMN decision table $T_d$ (over $\PCvoc_e \cup \PCvoc_d$) such that $f_e$ maps epistemic state $E$ to a structure $\struct$ iff for that epistemic state \EDMN table $T_d$ derives decision $d^\struct$.
    \label{the:edf-as-edmn}
\end{theorem}

\begin{proof}
    The proof is analogues to the one of Theorem \ref{the:edf-as-oel}. It is sufficient to observe that \EDMN can express any epistemic state as $\bigwedge_{c \in \voc_e} K[T_E][\bigvee_{v \in c^E} c = v]$ due to the new construct $\lor$ and epistemic interpretation of constraints, which are connected with conjunction ($\land$).
\end{proof}

Together theorems \ref{the:oed-as-edf} and \ref{the:edf-as-edmn} yield results similar to Corollary \ref{cor:odf-as-oel}.

\begin{corollary}
    Every \DMN decision function and optional decision function can be modeled in \EDMN.
\end{corollary}

\subsection{Decision Requirements Diagrams}
\DMN standard supports the composition of decision tables into a bigger decision model.
This is done by linking decision tables into dependency networks -- Decision Requirements Diagrams. 
Essential property of \DMN dependency networks is that they are free of loops. 
This property corresponds with the stratification property of \OEL theories, and hence it can be easily translated. 
A dependency graph in the \OEL translation is reflected in the referencing of the $K$ operator.

\section{Decision modeling fragment of \OEL}

The translation of \EDMN to \OEL finds the use of \textit{definitions} language construct.  
%Semantically definitions are not required for the purpose of translation (due to monotonicity), i.e., the same is possible to achieve with Clark's completion using only equivalences.
%However, t
The main motivation for the use of definitions is their syntactic property to clearly separate the body from the head, and definition parameters from defined concepts.
This feature is important for clear separation of environment from decision variables. 
Further, the \EDMN translational semantics strongly suggests that the body should consist only of epistemic formulas.
Imposing such syntactic constraint is trivial if decisions are modeled as definitions.

Hence, we believe that identifying this particular fragment of \OEL can be an important step towards a general KR language for decision modeling (contribution point 5 form the Introduction).
The following definition specifies such fragment that we call \EBD for Epistemic Body Definitions.
\begin{definition}
    The syntactic fragment \EBD (Epistemic Body Definitions) of \OEL requires theory to consists only of definitions composed of the rules of the following format: $\forall \bar{x} : p(\bar{t}) \leftarrow \phi$ or $\forall \bar{x} : f(\bar{t}) = v \leftarrow \phi$.
    Where $p$ is a predicate symbol and $f$ is a function symbol; $\phi$ is an FO formula, where each atomic formula occurs under the $K$ operator.
    Further, none of the defined symbols is allowed to be present in the body of a definition. Defined symbols constitute decision vocabulary $\voc_d$, and parameter symbols environment vocabulary $\voc_e$. 
\end{definition}

An important property of \EBD fragment is the guarantee of uniqueness of the decision. 
In other words, given an epistemic state about the environment variables, an \EBD theory has at most one model representing the solution, or it is unsatisfiable indicating an error in the decision table; formalized in the following theorem. 

\begin{theorem}
    \label{the:ebd-unique}
    Given an \OEL theory $\mathcal{T}$ consisting of two theories $T_E$ and $T$, such that $T_E$ is a standard first-order theory expressing an epistemic state $E$ of the environment variables, and $T$ is an \EBD theory (over $T_E$), then for any $T_E$ the following holds:
    \[Mod(T,T_E) = \{\} \lor Mod(T,T_E) = \{\struct\}\]
\end{theorem}

%\begin{proof}
%    First, let us observe that given $T_E$ each epistemic atom in $T$ is either true or false. 
%    This is an important property of Ordered Epistemic Logic.
%    Further, theory $T$ consists only of definitions whose rules are allowed only epistemic atoms in the body.
%    This means that the body of each rule can be interpreted given $T_E$, and evaluated to either true or false. 
%    On the other side, all the heads of the definitions are containing only objective atoms composed of decision variables.
%    Hence, each decision variable can take only the value derived by some rule -- those with the body evaluating to true. 
%    These values can be derived by multiple rules, but they all have to agree on the derived value.
%    Finally, if no rule applies or multiple applicable rules are deriving different decisions the theory has no models, i.e., it is unsatisfiable. 
%    This follows from the semantics of inductive definitions~\cite{Bruynooghe2016}.
%\end{proof}

Due to the space constraints the proof of Theorem \ref{the:ebd-unique} is not provided.
However, intuition is simple, it is a basic property of definitions to uniquely define a concept when all parameters are fixed, which is the case due to the epistemic constraint of the bodies.
The property described in Theorem \ref{the:ebd-unique} is a natural property of decision models.
Situations, when theory is unsatisfiable, correspond to hit policy and completeness property violation as described in \cite{calvanese2016semantics}. 
Inductive definitions naturally capture these properties making a good fit for the purpose of modeling decisions.
As indicated in the previous section, this property is not preserved if definitions are modeled as Clark's completion -- by using a set of equivalences.
Such theory would consist of axioms of the form $d = v \Leftrightarrow K[T_E][\phi]\land \dots \land K[T_E][\psi]$, where $d$ is a decision variable, and $v$ some decision value.
To see the problem, it is sufficient to imagine a decision model which is incomplete and furthermore, some decision values are not assigned by any rule.
Such theory, for a particular epistemic state that is not covered by any equivalence, would forbid the decision variable to take any of the mentioned values (left-hand-side), but it could take any of the values that are not assigned by any rule, leading to multiple possible models, or even worse to a unique but wrong one.

\section{Implementation}

Accompanying this paper, provided is the implementation\footnote{\url{https://zenodo.org/doi/10.5281/zenodo.10001954}} of the \OEL system as described in definitions \ref{def:oel} and \ref{def:oel-sem}, together with the implementation of \EDMN by translation to \OEL following definitions \ref{def:dmn-to-oel} and \ref{def:dmn-semantics}.
The system is implemented as a ``wrapper'' around the existing Knowledge Base System \IDPZT~\cite{carbonnelle2022idp}.
The \IDPZT system implements \FODOT language~\cite{Denecker2000,denecker2008logic} which is a first-order logic language extended with sorts, arithmetic, aggregates, inductive definitions, and other useful language constructs making it the perfect candidate for the purposes of this paper. 
The provided software addresses the contribution points 3 and 6 from the Introduction.

%\subsection{Solving problems using KB}
Providing \DMN with a model semantics, as it is done in this paper with \EDMN, gives rise to many possibilities.
It becomes possible to employ different inference methods to solve different problems with the same decision model.
In the case of decision making it might look like only one task is to be solved, given parameters make the decision.
However, it is totally realistic to ask the system what are the minimal requirements on the parameters for a certain decision to be made.
System developed for the purposes of this paper is capable of solving these problems.
A more detailed research on benefits of providing \DMN with a KB system\footnote{More generally on the KB systems \cite{Bruynooghe2016,de2018predicate,van2017kb,carbonnelle2022idp}.} is available in \cite{vandevelde2021leveraging,aerts2020tackling}.

\section{Examples}
This section provides an overview of \EDMN as a decision modeling language on a few examples. 

\paragraph{Greeting example} The extension of the running example~\cite{DMN-challenge:2016} modeled in \EDMN is presented in Figure \ref{fig:salutation-1}. 
The two added rules are expressing that in case gender is known to be ``Female'' and marital status is unknown, the salutation ``Lady'' is appropriate. 
Similarly, if the gender is unknown ``Customer'' can be used.
\begin{figure}[ht]
\centering
\small
\dmntable{Salutation}{A}{Gender, Marital status}{Salutation}
             {Male, -, Mr,
              Female, Single, Ms,
              Female, Married, Mrs,
              Female, $\neg K$, Lady,
              $\neg K$, -, Customer
             }
\vspace{-2em}
\caption{\DMN decision table - Greeting example}
\label{fig:salutation-1}
\end{figure}
Note that \EDMN allows a compact ``stratification'' of decisions based on the amount of knowledge possessed about the world.
The first three rules define decisions in cases of an exact epistemic state, while others are gradually increasing in ignorance.

\paragraph{Interview example} Originally presented in the paper of Gelfond and Lifschitz \cite{gelfond1991classical}, and later discussed in the context of \OEL in \cite{vlaeminck2012ordered}.
Here we present an extended version, where a student could get a grant, get rejected, or be invited to an interview based on their GPA and minority status.
High GPA students always get the grant, low GPA students never, while students with a fair GPA get the grant in case they are minorities and interviewed otherwise.
If the information about the student is not sufficient to make any of these decisions, the student should be interviewed.
In the original paper \cite{gelfond1991classical}, this last rule is of special interest as it employs both default and classical negation. 
Interesting is that default negation is sometimes interpreted epistemically~\cite{gelfond2014knowledge}, which is relevant for this work.
\begin{figure}[ht]
\centering
\small
\dmntable{Interview}{A}{GPA, Minority status}{Decision}
             {High, -, Approve,
              Fair, No, Interview,
              Fair, Yes, Approve,
              Fair, $\neg K$, Interview,
              Low, -, Reject,
              $\neg K$, -, Interview
             }
\vspace{-2em}
\caption{\DMN decision table - Interview example}
\label{fig:Approve}
\end{figure}
The \EDMN table presented in Figure \ref{fig:Approve} is correctly modeling the described decision protocol. 
This example is important as an Interview decision can be made for two reasons, in some exact epistemic states and also in an absence of knowledge.

\paragraph{More examples} More complex examples are available in the public repository\footnote{\url{https://zenodo.org/doi/10.5281/zenodo.10001954}}.
Some of the examples are: Loan application \cite{fox2018-handling-missing-data}, fabric recognition example from \cite{Maier1988-MAICWL}, Restaurant dilemma from \cite{peterson2017introduction}, and Adhesive selector (simplified version from \cite{vandevelde2022knowledge}).

\section{Related work}
The logic approach presented in this paper departs from the work on \OEL \cite{vlaeminck2012ordered}.
There it had been shown that stratified epistemic theories are useful in many contexts, including the expression of \emph{defaults} but also in the \textit{Interview example} which turns out to be a clean case of decision modeling in the context of incomplete knowledge. 
Our study contributes by taking a fragment of \OEL geared towards decision making, extending it with definitions, and proposing  a reasoning system for it. 
We argued that epistemic operators, stratified theories and definitions are essential aspects of decision modeling in the context of incomplete knowledge. 
To the best of our knowledge, there are no other scientific contributions investigating the epistemic nature of decisions from the KR perspective. 

\paragraph{Decision theory}
The standard setting of Decision Theory is at its core concerned with uncertainty.
These approaches then assume that some probabilities on the possible worlds are given, and a utility function measuring a benefit or cost of certain decisions.
The problem of decision-making then reduces to an optimization problem.
The research on uncertainty about probabilities, utility functions, or preference relations is common in this domain (e.g., \cite{WEBER198744,OZDEMIR2006349}). 
However, this approaches are completely differs from the one presented in this paper.
The approaches presented in \cite{peterson2017introduction} for making decisions in situations of ignorance about an exact state of affairs show that the problem tackled in this paper is recognized in Decision Theory.
Nevertheless, comparing the two approaches is difficult since Decision Theory is not concerned with the KR perspective of the problem.
In this sense, the approach presented in this paper provides a more natural modeling of purely epistemic decisions.
For example, in the \textit{Interview example}, a person is interviewed because some information about them is unknown, this information can not be naturally represented in utility function.

\paragraph{Default and other autoepistemic logics}
An interesting connection to this work is that of non-monotonic reasoning in AI.
\citeauthor{reiter1981closed} introduced default logic~\cite{reiter1981closed} for reasoning with defaults.
While autoepistemic logic~\cite{moore1985possible} of \citeauthor{moore1985possible} is a logic for expressing defeasible inference for defaults.
The relation between the two is investigated in \cite{denecker2011reiter}.
\citeauthor{levesque1990all} further worked on extension of the system with ``only knowing'' constructs~\cite{levesque1990all}.
\citeauthor{konolige1982circumscriptive} developed an inference about ignorance called circumscriptive
ignorance~\cite{konolige1982circumscriptive}.
Ordered Epistemic Logic \OEL is related to defaults and autoepistemic logic, however, the sticking fact is that in this paper it is used for a totally different purpose.

Ordered Epistemic Logic is less expressive than the other formalisms mentioned above (these results are proven for the \OEL and Autoepistemic logic in \cite{vlaeminck2012ordered}).
However, this is a desired property, as it results in a less complex system.
This reduced expressiveness and complexity comes from the stratification of the theories (i.e., there are no self-referential statements in \OEL).
This makes \OEL more suitable for the purpose of modeling decisions as they are always stratified.

\paragraph{Other KR languages} 
The KR languages based on the Logic Programming (LP)~\cite{lloyd2012foundations} paradigm are playing an important role in symbolic knowledge representation.
The relation to this work becomes interesting due to many interpretations of logic programs\footnote{\textit{Negation as failure} causes problems with program interpretations, more in \cite{denecker2017negation}.} among which is an epistemic interpretation~\cite{gelfond2014knowledge} in the case of Answer Set Programming (ASP)~\cite{vladimir2008answer}.
Nevertheless, ASP failed in expressing an \textit{Interview example} (in case of incomplete knowledge) as pointed out by \citeauthor{kahl2020language} in~\cite{kahl2020language} where epistemic operator is introduced in ASP to resolve this issue.
However, the stratification of decisions can cause problems in ASP as demonstrated in the following example.
Suppose that the decision of an interview is a two-step process, the first step is as before, and in the second step, if the total number of persons to be interviewed does not exceed a certain number other candidates (with low GPA) can be added to fulfill the number. 
An ASP epistemic specification would result in a circular definition, causing a problem, while \OEL stratification would provide a simple model of this decision process.

\paragraph{Existing FO translations of \DMN} 
\citeauthor{calvanese2016semantics} present in~\cite{calvanese2016semantics,calvanese2019semantic} translational semantics of \DMN. 
These translations do not support reasoning with missing values which is the aim of this paper.

\paragraph{Machine learning}
In non-symbolic AI, algorithms for training decision models on data with missing values are active field of research (e.g., \cite{kotsiantis2013decision}). 
This methods are not related to our work as they are not concern with KR modeling aspects of decisions.

\paragraph{Adhesive selector} In the work~\cite{vandevelde2022knowledge} the authors have tackled the real industry problem of adhesive selection process.
They identified the lack of epistemic operators as a big issue for modeling this complex decision process.
The logic presented in this paper would allow for natural modeling of this problem.

\section{Conclusion}
The problem of making decisions with missing information about the state of affairs was investigated in this paper. 
We observed that in most decision modeling formalisms there are no means for expressing an epistemic state of the agent, and yet uncertainty is arising from the agent's ignorance of the state of affairs.
We show the importance and expressive power of epistemic logic for tackling this issue. 
Further, we propose and implement a new formalism for modeling decisions based on existing standards.
Finally, we believe this work lays a solid foundation for future work in this field. 

\section{Acknowledgements} Thanks to Tobias Reinhard for the discussions and for reading this paper.

%\section{Acknowledgments}
%This work was partially supported by the Flemish Government under the ``Onderzoeksprogramma Artificiële Intelligentie (AI) Vlaanderen''.

\bibliographystyle{ACM-Reference-Format}
\bibliography{biblio}
\end{document}